\documentclass[conference]{IEEEtran}
\IEEEoverridecommandlockouts
\usepackage{cite}
\usepackage{amsmath,amssymb,amsfonts}
\usepackage{algorithmic}
\usepackage{graphicx}
\usepackage{textcomp}
\usepackage{xcolor}
\def\BibTeX{{\rm B\kern-.05em{\sc i\kern-.025em b}\kern-.08em
    T\kern-.1667em\lower.7ex\hbox{E}\kern-.125emX}}

\usepackage{url}
\usepackage{booktabs}
\usepackage{colortbl}
\usepackage{makecell}
\usepackage{comment}
\usepackage{algorithm}
\usepackage{eqparbox}
\usepackage{amsthm}
\usepackage{color}
\usepackage{array}
\usepackage{bm}
\usepackage[nolist,nohyperlinks]{acronym}
\usepackage{cuted}
\usepackage{multirow}

\newtheorem{definition}{Definition}
\newtheorem{proposition}{Proposition}
\newtheorem{claim}{Claim}

\acrodef{ML}{Machine Learning}
\acrodef{DL}{Deep Learning}
\acrodef{FL}{Federated Learning}
\acrodef{DNN}{Deep Neural Network}
\acrodef{GDPR}{General Data Protection Regulation}
\acrodef{DLG}{Deep Leakage from Gradients}
\acrodef{iDLG}{Improved \ac{DLG}}
\acrodef{IG}{Inverting Gradient}
\acrodef{GGL}{Generative Gradient Leakage}
\acrodef{PRECODE}{Privacy Enhancing Module}
\acrodef{CMA-ES}{Covariance Matrix Adaptation Evolution Strategy}
\acrodef{GRNN}{Generative Regression Neural Network}
\acrodef{GAN}{Generative Adversarial Network}
\acrodef{CNN}{Convolutional Neural Network}
\acrodef{DP}{Differential Privacy}
\acrodef{HE}{Homomorphic Encryption}
\acrodef{MPC}{Secure Multi-Party Computation}
\acrodef{MSE}{Mean Square Error}
\acrodef{PSNR}{Peak Signal-to-Noise Ratio}
\acrodef{LPIPS}{Learned Perceptual Image Patch Similarity}
\acrodef{SSIM}{Structure Similarity Index Measure}
\acrodef{ReLU}{Rectified Linear Unit}
\acrodef{CE}{Cross Entropy}
\acrodef{BN}{Batch Normalization}
\acrodef{FedAvg}{Federated Averaging}
\acrodef{FedAdp}{Federated Adaptive Weighting}
\acrodef{SGD}{Stochastic Gradient Descent}
\acrodef{GLAUS}{Gradient Leakage Attack Using Sampling}
\acrodef{UGS}{Unbiased Gradient Sampling-based}
\acrodef{MMS}{MinMax Sampling}

\begin{document}

\title{Gradients Stand-in for Defending Deep Leakage in Federated Learning\\
}

\author{\IEEEauthorblockN{1\textsuperscript{st} Yi Hu}
\IEEEauthorblockA{\textit{Department of Computer Science} \\
\textit{Swansea University}\\
Swansea, United Kingdom \\
845700@swansea.ac.uk}
~\\
\and
\IEEEauthorblockN{2\textsuperscript{nd} Hanchi Ren*}
\IEEEauthorblockA{\textit{Department of Computer Science} \\
\textit{Swansea University}\\
Swansea, United Kingdom \\
hanchi.ren@swansea.ac.uk}
*Corresponding author
~\\
\and
\IEEEauthorblockN{3\textsuperscript{th} Chen Hu}
\IEEEauthorblockA{\textit{Department of Computer Science} \\
\textit{Swansea University}\\
Swansea, United Kingdom \\
2100552@swansea.ac.uk}

\and
\IEEEauthorblockN{4\textsuperscript{th} Yiming Li}
\IEEEauthorblockA{\textit{Department of Computer Science} \\
\textit{Swansea University}\\
Swansea, United Kingdom \\
946802@swansea.ac.uk}

\and
\IEEEauthorblockN{5\textsuperscript{rd} Jingjing Deng}
\IEEEauthorblockA{\textit{Department of Computer Science} \\
\textit{Durham University}\\
Durham, United Kingdom \\
jingjing.deng@durham.ac.uk}
~\\

\and
\IEEEauthorblockN{6\textsuperscript{th} Xianghua Xie}
\IEEEauthorblockA{\textit{Department of Computer Science} \\
\textit{Swansea University}\\
Swansea, United Kingdom \\
x.xie@swansea.ac.uk}

}

\maketitle

\begin{abstract}
Federated Learning (FL) has become a cornerstone of privacy protection, shifting the paradigm towards localizing sensitive data while only sending model gradients to a central server. This strategy is designed to reinforce privacy protections and minimize the vulnerabilities inherent in centralized data storage systems. Despite its innovative approach, recent empirical studies have highlighted potential weaknesses in FL, notably regarding the exchange of gradients. In response, this study introduces a novel, efficacious method aimed at safeguarding against gradient leakage, namely, ``AdaDefense". Following the idea that model convergence can be achieved by using different types of optimization methods, we suggest using a local stand-in rather than the actual local gradient for global gradient aggregation on the central server. 
This proposed approach not only effectively prevents gradient leakage, but also ensures that the overall performance of the model remains largely unaffected. Delving into the theoretical dimensions, we explore how gradients may inadvertently leak private information and present a theoretical framework supporting the efficacy of our proposed method. Extensive empirical tests, supported by popular benchmark experiments, validate that our approach maintains model integrity and is robust against gradient leakage, marking an important step in our pursuit of safe and efficient FL.
\end{abstract}


\section{Introduction}
\label{sec:introduction}

\begin{figure}[t!]
    \centering
    \includegraphics[width=0.99\linewidth]{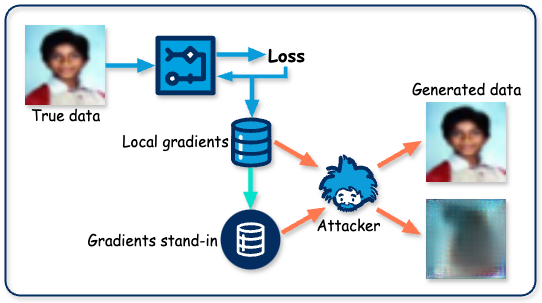}
    \caption{Illustration of AdaDefense against gradient leakage}
    \label{fig:illustration}
\end{figure}

The remarkable advancements in \ac{DNN} across diverse \ac{DL} tasks have significantly impacted everyday life. This progress largely depends on the availability of extensive training datasets. In recent years, there has been a surge in data creation, including much sensitive private data that cannot be shared due to privacy concerns. The implementation of the \ac{GDPR} in Europe underscores this by aiming to protect personal data integrity and regulate data exchange. Consequently, \ac{DL} models leveraging such datasets have enhanced the capabilities of numerous applications. Given these developments, exploiting all available data for training models in a legally compliant manner has gained substantial traction. This ensures the utility of expansive datasets while adhering strictly to privacy regulations.

For this purpose, \ac{FL} \cite{mcmahan2016communication, konevcny2016federated, konevcny2016federatedlearning, mcmahan2017federated} is employed as a decentralized model training framework where the central server accumulates local gradients from clients without requiring the exchange of local data. This approach ensures that data remains at its source, ostensibly preserving privacy because only gradient information is shared. Traditionally, these gradients have been assumed secure for sharing, making \ac{FL} a prime application in privacy-sensitive scenarios. Nonetheless, recent studies have exposed vulnerabilities within this gradient-sharing framework. L. Zhu, Z. Liu and S. Han~\cite{zhu2019deep} introduced \ac{DLG}, a technique for reconstructing the original training data by iteratively optimizing randomly initialized input images and labels to match the shared gradients, rather than updating the model parameters directly. However, the performance of \ac{DLG} is limited by factors such as the training batch size and the resolution of images, which can lead to instability in data recovery. To address these instabilities, J. Geiping, H. Bauermeister, H. Dr{\"o}ge and M. Moeller~\cite{geiping2020inverting} proposed \ac{IG}, which utilizes a magnitude-invariant cosine similarity as the loss function to enhance the stability of training data reconstruction. This method has proven effective in recovering high-resolution images ($224*224$ pixels) from gradients of large training batches (up to 100 samples). Further refining the approach, \ac{iDLG} by B. Zhao, K.R. Mopuri and H. Bilen~\cite{zhao2020idlg} simplifies the label recovery process within \ac{DLG} by analytically deriving the ground-truth labels from the gradients of the loss function relative to the \emph{Softmax} layer outputs, thus improving the precision of the reconstructed data. Expanding on generative methods, H. Ren, J. Deng and X. Xie~\cite{ren2022grnn} developed \ac{GRNN}, which integrates two generative model branches: one, a \ac{GAN}-based branch for creating synthetic training images, and another, a fully-connected layer designed to generate corresponding labels. This model facilitates the alignment of real and synthetic gradients to reveal the training data effectively. Recently, Z. Li, J. Zhang, L. Liu and J. Liu~\cite{li2022auditing} introduced \ac{GGL}, a novel method that also utilizes a pre-trained \ac{GAN} to generate fake data. This approach leverages the \ac{iDLG} concept to deduce true labels and adjusts the \ac{GAN} input sequence based on the gradient matching, thereby producing fake images that closely resemble the original training images. The paper~\cite{yang2023reveal} introduces a novel attack method named \ac{GLAUS} which targets \ac{UGS} secure aggregation methods in \ac{FL}. These \ac{UGS}, e.g. \ac{MMS}~\cite{zhao2022minmax}, methods are designed to enhance security by using unbiased random transformations and gradient sampling to prevent direct access to the real gradients during model training, thereby obscuring private client data. These developments underscore the ongoing need to enhance the security measures in \ac{FL} systems against sophisticated attacks that seek to exploit gradient information, thus compromising the privacy of sensitive data. The continuous evolution of defensive techniques is crucial to safeguarding data integrity in federated environments.

These issues prompted us to evaluate the reliability of the \ac{FL} system. Concerns about gradient leakage have triggered various defense strategies, such as gradient perturbation~\cite{zhu2019deep, yang2022accuracy, sun2021ldp, sun2021soteria, ren2022grnn}, data obfuscation or sanitization~\cite{hasan2016effective, chamikara2018efficient, chamikara2020efficient, lee2021digestive, chamikara2021privacy}, along with other techniques~\cite{bu2020deep, li2020privacy, yadav2020differential, wei2021gradient, ren2024fedboosting}. Nonetheless, these methods generally involve compromises between privacy protection and computational efficiency, often requiring substantial computational resources due to the complexities of encryption technologies. The study in~\cite{wei2020framework} delves into the factors in an \ac{FL} system that could influence gradient leakage, such as batch size, image resolution, the type of activation function used, and the number of local iterations prior to the exchange of gradients. Various studies corroborate these findings; for instance, the \ac{DLG} method indicates that activation functions need to be bi-differentiable. \ac{IG} techniques can reconstruct images with resolutions up to $224 * 224$, whereas \ac{GRNN} are compatible with batch sizes of 256 and image resolutions of $256 * 256$, also assessing how the frequency of local iterations might impact privacy leakage. Moreover, recent advancements by D. Scheliga, P. M{\"a}der and M. Seeland~\cite{scheliga2022precode} introduced an innovative model extension, named \ac{PRECODE}, designed specifically to enhance privacy safeguards against potential leakages in \ac{FL} systems. This adaptability highlights the evolving trend of \ac{FL} security measures that seek to enhance the balance between performance and privacy without imposing excessive computational requirements.

In traditional model training strategies, various optimization algorithms can alter the low-level representations of gradients while preserving their high-level representations. For instance, unlike \ac{SGD} which utilizes raw gradients for model updates~\cite{ruder2016overview}, Adam~\cite{kingma2014adam} adjusts gradient values using adaptive estimates based on first-order and second-order moments. This work introduces an alternative approach by employing gradients stand-in for global gradient aggregation, aimed at preventing the transmission of private information to the parameter server (refer to Fig.~\ref{fig:illustration}). After an epoch of local training, the local gradients are modified through the Adam optimization algorithm, which utilizes adaptive estimates of lower-order moments. Critical information, such as first-order and second-order moments, is retained locally, thereby rendering it impossible for an adversary to execute the forward-backward procedure without access to this data. Consequently, reconstructing local training data from the shared gradients becomes infeasible within the \ac{FL} framework. The gradients stand-in calculated in this manner not only mitigate the risk of gradient leakage but also preserve model performance. Additionally, the simplicity and computational efficiency of the method, make it particularly attractive. The versatility of the approach is evident as it does not impose constraints on the model architecture, the local training optimization strategy, the dataset, or the \ac{FL} aggregation technique. Comprehensive evaluations on various benchmark networks and datasets have demonstrated that our method, termed AdaDefense, effectively addresses the identified privacy concerns. The implementation of AdaDefense is available on Github\footnote{\url{https://github.com/Rand2AI/AdaDefense}}.

The structure of this paper is organized as follows: Section~\ref{sec:rw} reviews related work on gradient leakage attacks along with existing defense mechanisms. Section~\ref{sec:pm} details the proposed method for defending against gradient leakage. Experimental results are presented in Section~\ref{sec:er}, where we compare the efficacy of our proposed AdaDefense with state-of-the-art attack methods. Finally, Section~\ref{sec:cc} concludes the paper and outlines future research directions.

\section{Related Work}
\label{sec:rw}

\subsection{Attack Methods}

Gradient leakage attacks pose a significant threat to privacy in both centralized and collaborative \ac{DL} systems by potentially exposing private training data through leaked gradients. Such attacks are particularly prevalent in centralized systems, often through membership inference attacks as discussed by \cite{shokri2017membership,truex2018towards, truex2019demystifying,choquette2021label,zhang2022label,xie2024survey}. 

The work~\cite{zhu2019deep} was pioneer in examining data reconstruction from leaked gradients within collaborative systems, an effort furthered by \cite{geiping2020inverting} which introduced an optimization-based method to enhance the stability of such attacks. They utilized magnitude-invariant cosine similarity measurements for the loss function and demonstrated that incorporating prior knowledge could significantly augment the efficiency of gradient leakage attacks. Expanding on these findings, the work~\cite{jeon2021gradient} argued that gradient information alone might not suffice for revealing private data, thus they proposed the use of a pre-trained model, GIAS, to facilitate data exposure. In~\cite{yin2021see}, the authors found that in image classification tasks, ground-truth labels could be easily discerned from the gradients of the last fully-connected layer, and that \ac{BN} statistics markedly enhance the success of gradient leakage attacks by disclosing high-resolution private images. 

Alternative approaches involve generative models, as explored in \cite{hitaj2017deep}, which employed a \ac{GAN}-based method for data recovery that replicates the training data distribution. Similarly, the work~\cite{wang2019beyond} developed mGAN-AI, a multitask discriminator-enhanced \ac{GAN}, to reconstruct private information from gradients. \ac{GRNN}, introduced~\cite{ren2022grnn}, is capable of reconstructing high-resolution images and their labels, effectively handling large batch sizes. Likewise, in the work of \ac{GGL}~\cite{li2022auditing}, a \ac{GAN} with pre-trained and fixed weights was used. \ac{GGL} differs in its use of the \ac{CMA-ES} as the optimizer, reducing variability in the generated data, which, while not exact replicas, closely resemble the true data. This characteristic significantly enhances \ac{GGL}'s robustness, allowing it to counteract various defensive strategies including gradient noising, clipping, and compression. This series of developments highlights the dynamic and evolving nature of combating gradient leakage attacks in \ac{DL}.

\ac{GLAUS}~\cite{yang2023reveal} demonstrates a significant vulnerability in \ac{UGS} methods by showing that they can still be susceptible to gradient leakage attacks. This is done by reconstructing private data points with considerable accuracy despite the supposed robust security measures of \ac{UGS}. \ac{GLAUS} circumvents the safeguards by approximating the gradient using leaked indices and signs from the \ac{UGS}'s crafted random transformations. It effectively reduces the security of \ac{UGS} frameworks to that of basic \ac{FL} models without additional protections. It is capable of adaptively inferring the gradient without needing the exact gradient, utilizing an approximate gradient reconstructed through several steps. These include narrowing the gradient search range, estimating the magnitude of each gradient value, and revising the gradient signs. This method showcases that even when real gradients are not directly accessible, sensitive data can still be reconstructed, posing a serious security risk.

\subsection{Defense Methods}

Numerous strategies~\cite{xie2024survey} have been developed to safeguard private data against potential leakage via gradient sharing in \ac{FL}. Techniques such as gradient perturbation, data obfuscation or sanitization, \ac{DP}, \ac{HE}, and \ac{MPC} have been employed to protect both the private training data and the publicly shared gradients~\cite{li2020privacy}. 

The work in~\cite{zhu2019deep} assessed the efficacy of Gaussian and Laplacian noise types in protecting data, discovering that the magnitude of the noise's distribution variance is crucial, with a variance threshold above $10^{-2}$ effectively mitigating leakage attacks at the cost of significant performance degradation. The work~\cite{chamikara2021privacy} introduced a data perturbation method that maintains model performance while securing the privacy of training data. This approach transforms the input data matrix into a new feature space through a multidimensional transformation, applying variable scales of transformation to ensure adequate perturbation. However, this technique depends on a centralized server for generating global perturbation parameters and may distort the structural information in image-based datasets. The method proposed in~\cite{wei2021gradient} implemented \ac{DP} to introduce noise into each client's training dataset using a per-example-based \ac{DP} method, termed Fed-CDP. They proposed a dynamic decay method for noise injection to enhance defense against gradient leakage and improve inference performance. Despite its effectiveness in preventing data reconstruction from gradients, this method significantly reduces inference accuracy and incurs high computational costs due to the per-sample application of \ac{DP}. 

Both \ac{PRECODE}~\cite{scheliga2022precode} and FedKL~\cite{ren2023gradient} aim to prevent input information from propagating through the model during gradient computation. \ac{PRECODE} achieves this by incorporating a probabilistic encoder-decoder module ahead of the output layer, which normalizes the feature representations and significantly hinders input data leakage through gradients. This process involves encoding the input features into a sequence, normalizing based on calculated mean and standard deviation values, and then decoding into a latent representation that feeds into the output layer. While effective, the additional computational overhead from the two fully-connected layers required by \ac{PRECODE} limits its applicability to shallow \ac{DNN}s due to high computational costs. In contrast, FedKL method introduces a key-lock module that manages the weight parameters with a hyper-parameter controlling the input dimension and an output dimension optimized for specific architectures; 16 for \emph{ResNet-20} and \emph{ResNet-32} on CIFAR-10 and CIFAR-100, and 64 for \emph{ResNet-18} and \emph{ResNet-34} on the ILSVRC2012 dataset. This configuration not only secures the gradients against leakage but also maintains manageable computational demands, making it feasible for more complex \ac{DNN}s. 

Diverging from the above methods, our newly proposed AdaDefense strategically integrates the Adam algorithm as a plugin component. This integration processes the original local gradients to produce a gradients stand-in, which is then used for global gradient aggregation. The gradients stand-in preserves the high-level latent representations of the original gradients, thereby ensuring minimal impact on model performance following each aggregation round. Furthermore, since the complex computation of gradients stand-in and both the first-order and second-order moments are retained locally, it becomes infeasible for a malicious attacker to reconstruct private training data from the gradients stand-in. This enhancement not only fortifies the privacy safeguards within \ac{FL} environments but also maintains the integrity and effectiveness of the learning process.

\section{Methodology}
\label{sec:pm}

In the initial subsection, we provide an overview of how gradient leakage can expose private training data. Subsequently, we performed a mathematical analysis to demonstrate the efficacy of our proposed gradients stand-in in defending leakage attacks. This analysis explains how the gradients stand-in protects sensitive information while maintaining model performance. Finally, we present a comprehensive introduction to the proposed AdaDefense method.

\subsection{Gradient Leakage}

In the context of \ac{ML}, particularly in training neural networks, the gradient points in the direction in which the function (often a loss or cost function) most quickly increases. Finding the function's minimum involves moving in the direction opposite to the gradient, a process known as gradient descent.

\begin{definition}
    \textbf{Gradient:} Consider a function $\mathbf{f}: \mathbb{R}^n \rightarrow \mathbb{R}^m$, which maps a vector $\mathbf{x}$ in $n$-dimensional space to a vector $\mathbf{y}$ in $m$-dimensional space, defined by $\mathbf{f}(\mathbf{x}) = \mathbf{y}$ where $\mathbf{x} \in \mathbb{R}^n$ and $\mathbf{y} \in \mathbb{R}^m$. The gradient of $\mathbf{f}$ with respect to $\mathbf{x}$ is represented by the Jacobian matrix as follows:

    \begin{align}
        \frac{\partial \mathbf{f}}{\partial \mathbf{x}} =
        \begin{pmatrix}
            \frac{\partial f_1}{\partial x_1} & \cdots & \frac{\partial f_1}{\partial x_n} \\
            \vdots & \ddots & \vdots \\
            \frac{\partial f_m}{\partial x_1} & \cdots & \frac{\partial f_m}{\partial x_n}
        \end{pmatrix} \nonumber
    \end{align}
    This matrix consists of partial derivatives where the element in the $i$-th row and $j$-th column, $\frac{\partial f_i}{\partial x_j}$, represents the partial derivative of the $i$-th component of $\mathbf{y}$ with respect to the $j$-th component of $\mathbf{x}$.
\end{definition}

The \emph{Chain Rule} is a rule in calculus that describes the derivative of a composite function. In simple terms, if a variable $z$ depends on $y$, and $y$ depends on $x$, then $z$, indirectly through $y$, depends on $x$, and the chain rule helps in finding the derivative of $z$ with respect to $x$. If $z = f(y)$ and $y = g(x)$, then:

\begin{align}
    \frac{dz}{dx} = \frac{dz}{dy} \cdot \frac{dy}{dx} \nonumber
\end{align}

Many recent studies have highlighted the vulnerability of gradient data to leakage attacks. In their research, the work~\cite{geiping2020inverting} demonstrated that for fully-connected layers, the gradients of the loss with respect to the outputs can reveal information about the input data. By employing the \emph{Chain Rule}, it is possible to reconstruct the inputs of a fully-connected layer using only its gradients, independent of the gradients from other layers. Expanding on this, in FedKL~\cite{ren2023gradient}, the authors explored this vulnerability in convolutional and \ac{BN} layers across typical supervised learning tasks. They provided a detailed analysis of how gradients carry training data, which can be exploited by an attacker to reconstruct this data. Their results confirmed a strong correlation between the gradients and the input training data across linear and non-linear neural networks, including \ac{CNN} and \ac{BN} layers. Ultimately, they concluded that in image classification tasks, the gradients from the fully-connected, convolutional, and \ac{BN} layers contain ample information about the input data and the true labels. This richness of information enables an attacker to reconstruct both the inputs and labels by regressing the gradients. We can conclude that:
\begin{proposition}
    In image classification tasks, the gradients from the fully connected, convolutional and \ac{BN} layers contain a great deal of information about the input data and the actual labels. These details allow an attacker to approximate these gradients and effectively reconstruct the corresponding inputs and labels.
\end{proposition}

\subsection{Theoretical Analysis on Gradients Stand-in}

\noindent\textbf{Model Performance:}
In this part, we explore the influence of using gradients stand-in on model performance, and provide justification for employing gradients stand-in during global gradient aggregation in a \ac{FL} context. We begin by reviewing the Adam~\cite{kingma2014adam} optimization algorithm, a method of \ac{SGD} that incorporates momentum concepts. Each iteration in Adam involves computing the first-order and second-order moments of the gradients, followed by calculating their exponential moving averages to update the model parameters. This approach merges the strengths of the Adagrad \cite{duchi2011adaptive} algorithm, which excels in handling sparse data, with those of the RMSProp \cite{hinton2012neural} algorithm, designed to manage non-smooth data effectively. Collectively, these features enable Adam to deliver robust performance across a wide range of optimization problems, from classical convex formulations to complex \ac{DP} tasks. Further details on the Adam algorithm are presented in Algorithm~\ref{alg:adam}.

\begin{algorithm}[ht!]
    \caption{Adam Optimization Algorithm}
    \label{alg:adam}
    \begin{algorithmic}[1]
        \REQUIRE Initial learning rate $\alpha$
        \REQUIRE Exponential decay rates for moment estimates $\beta_1$, $\beta_2$ $\in [0,1]$
        \REQUIRE Small constant for numerical stability $\epsilon$
        \REQUIRE Maximum number of iterations $I$
        \STATE Initialize weights $\omega_0$
        \STATE Initialize the first-order moment vector $m_0 \leftarrow 0$
        \STATE Initialize the second-order moment vector $v_0 \leftarrow 0$
        \STATE Initialize the time-step $t \leftarrow 0$
        \FOR{$t = 1$ to $I$}
            \STATE $g_t \leftarrow \nabla_\omega \mathcal{L}(\omega_{t-1})$ $\hfill\blacktriangleright$Get gradients w.r.t. stochastic objective at time-step $t$
            \STATE $m_t \leftarrow \beta_1 \cdot m_{t-1} + (1 - \beta_1) \cdot g_t$ $\hfill\blacktriangleright$Update biased first-order moment estimate
            \STATE $v_t \leftarrow \beta_2 \cdot v_{t-1} + (1 - \beta_2) \cdot g_t^2$ $\hfill\blacktriangleright$Update biased second-order moment estimate
            \STATE $\hat{m}_t \leftarrow \frac{m_t}{1 - \beta_1^t}$ $\hfill\blacktriangleright$Compute bias-corrected first-order moment estimate
            \STATE $\hat{v}_t \leftarrow \frac{v_t}{1 - \beta_2^t}$ $\hfill\blacktriangleright$Compute bias-corrected second-order moment estimate
            \STATE $\omega_t \leftarrow \omega_{t-1} - \alpha \cdot \frac{\hat{m}_t}{\sqrt{\hat{v}_t} + \epsilon}$ $\hfill\blacktriangleright$Update parameters
        \ENDFOR
        \RETURN $\omega_t$ $\hfill\blacktriangleright$Output the optimized parameters
    \end{algorithmic}
\end{algorithm}

In the local training phase of \ac{FL}, we concentrate on the aggregated gradients from each training round rather than the gradients from individual iterations. Specifically, the local gradients that are transmitted to the global server represent the cumulative sum of all gradients produced during the local training iterations. Common \ac{FL} aggregation methods, such as \ac{FedAvg}~\cite{mcmahan2016communication}, FedAdp~\cite{wu2021fast}, and FedOpt~\cite{reddi2020adaptive}, employ these cumulative local gradients for global model updates. This approach effectively outlines the model's convergence path and direction within an \ac{FL} framework. Moreover, studies such as Adam~\cite{kingma2014adam} optimization algorithm has demonstrated superior convergence properties compared to traditional \ac{SGD}, underscoring its widespread adoption in the \ac{DL} domain. Based on these insights, we propose a novel method of representing local gradients using the Adam algorithm to enhance privacy and maintain model efficacy in \ac{FL}. Experimental outcomes discussed in Section~\ref{sec:er} confirm the viability of using these modified gradients for global aggregation, thereby supporting the integrity and performance of the \ac{FL} model.

\noindent\textbf{Leakage Defense:}
In the FedKL~\cite{ren2023gradient}, it was demonstrated that private training data leakage through gradients sent to a global server is a significant concern. The authors of FedKL introduced a novel approach involving a key-lock pair to generate shift and scale parameters in the \ac{BN} layer, which are typically trainable within the model. Crucially, in FedKL, these generated parameters are retained on the local clients. The primary aim of FedKL is to sever the transmission of private training data via gradients, thereby safeguarding against potential leakage by malicious servers. This method involves modifications to the network architecture, including additional layers that are responsible for generating the shift and scale parameters. While these changes have minimal impact on model performance, they significantly reduce the system's efficiency due to the increased computational overhead. In our work, we adhere to the principle of preventing the propagation of private training data through gradients. However, unlike FedKL, we achieve this without altering the network architecture or adding extra training layers, thus avoiding additional computational burdens. The proposed approach maintains system efficiency while still protecting against data leakage.

\begin{claim}
\label{claim:1}
    The gradients stand-in transmitted to the global server no longer contains information from which details of the private training data can be inferred.
\end{claim}
\begin{proof}
    Assuming $g_r$ is the local gradients in the $r$-th training round and $\hat{g}_r$ is the gradients stand-in for global gradients aggregation. By applying the Adam in Algorithm~\ref{alg:adam}, we have:
    \begin{align}
        m_r &= \beta_1 m_{r-1} + (1 - \beta_1) g_r \label{eq:m}\\
        v_r &= \beta_2 v_{r-1} + (1 - \beta_2) g_r^2 \label{eq:v}\\
        \hat{m}_r &= \frac{m_r}{1 - \beta_1^r} \label{eq:mm}\\
        \hat{v}_r &= \frac{v_r}{1 - \beta_2^r} \label{eq:vv}\\
        \hat{g}_r &= \frac{\hat{m}_r}{\sqrt{\hat{v}_r} + \epsilon} \label{eq:gg}
    \end{align}
    where $\beta_1$ and $\beta_2$ are set to $0.9$ and $0.999$, separately; the values in $g_0$ are initialized to all zeros; $g_r \in \mathbb{R}^D$, $m_r \in \mathbb{R}^D$, $v_r \in \mathbb{R}^D$, $\hat{m}_r \in \mathbb{R}^D$, $\hat{v}_r \in \mathbb{R}^D$ and $\hat{g}_r \in \mathbb{R}^D$, $D$ is the dimension of the parameter space, $m_r$ is the biased first-order moment estimate in the $r$-th training round, $v_r$ is the biased second-order moments estimate, $\hat{m}_r$ is the bias-corrected first-order moment estimate and $\hat{v}_r$ is the bias-corrected second-order moment estimate. The derivative of the gradients stand-in, $\hat{g}_r$, w.r.t. the original local gradients, $g_r$, can be expressed as:
    \begin{align}
        \frac{\partial{\hat{g}_r}}{\partial g_r} &= \frac{\partial{\hat{g}_r}}{\partial{\hat{m}_r}} \frac{\partial{\hat{m}_r}}{\partial g_r} +  \frac{\partial{\hat{g}_r}}{\partial{\hat{v}_r}} \frac{\partial{\hat{v}_r}}{\partial g_r}
    \end{align}
    Given \eqref{eq:m} and \eqref{eq:v}, differentiating $m_r$ and $v_r$ w.r.t. $g_r$:
    \begin{align}
        \frac{\partial m_r}{\partial g_r} &= 1 - \beta_1 \\
        \frac{\partial v_r}{\partial g_r} &= 2(1 - \beta_2)g_r
    \end{align}
    Given \eqref{eq:mm} and \eqref{eq:vv}, differentiating $\hat{m}_r$ and $\hat{v}_r$ w.r.t. $g_r$:
    \begin{align}
        \frac{\partial \hat{m}_r}{\partial g_r} &= \frac{\frac{\partial m_r}{\partial g_r}}{1 - \beta_1^r} = \frac{1 - \beta_1}{1 - \beta_1^r} \\
        \frac{\partial \hat{v}_r}{\partial g_r} &= \frac{\frac{\partial v_r}{\partial g_r}}{1 - \beta_2^r} = \frac{2(1 - \beta_2)g_r}{1 - \beta_2^r}
    \end{align}
    Given \eqref{eq:gg}, using the quotient rule, where $u = \hat{m}_r$ and $z = \sqrt{\hat{v}_r} + \epsilon$:
    \begin{align}
        \frac{\partial{\hat{g}_r}}{\partial g_r} &= \frac{\frac{\partial \hat{m}_r}{\partial g_r} z - u \frac{\partial}{\partial g_r} (\sqrt{\hat{v}_r} + \epsilon)}{z^2}
    \end{align}
    Calculating $\frac{\partial}{\partial g_r} (\sqrt{\hat{v}_r} + \epsilon)$:
    \begin{align}
        \frac{\partial \sqrt{\hat{v}_r}}{\partial g_r} &= \frac{1}{2\sqrt{\hat{v}_r}} \cdot \frac{\partial \hat{v}_r}{\partial g_r} \nonumber\\
        &= \frac{1}{2\sqrt{\hat{v}_r}} \cdot \frac{2(1 - \beta_2)g_r}{1 - \beta_2^r} \nonumber\\
        &= \frac{(1 - \beta_2) g_r}{\sqrt{\hat{v}_r} (1 - \beta_2^r)}
    \end{align}
    Then, differentiating $\hat{g}_r$ w.r.t. $g_r$:
    \begin{align}
        \frac{\partial{\hat{g}_r}}{\partial g_r} &= \frac{(\frac{1 - \beta_1}{1 - \beta_1^r})(\sqrt{\hat{v}_r} + \epsilon) - \hat{m}_r \cdot \frac{(1 - \beta_2)g_r}{\sqrt{\hat{v}_r}(1-\beta_2^r)}}{(\sqrt{\hat{v}_r} + \epsilon)^2} \label{eq:dgg}
    \end{align}
    Substitute $\hat{m}_r$ and $\hat{v}_r$ in \eqref{eq:dgg} referring to \eqref{eq:m}, \eqref{eq:v}, \eqref{eq:mm} and \eqref{eq:vv}. To simplify the equation. We define:
    \begin{align}
        V &= \sqrt{\hat{v}_r} \nonumber\\
        &= \sqrt{\frac{\beta_2 v_{r-1} + (1-\beta_2) g_r^2}{1-\beta_2^r}} \label{eq:vvv}
    \end{align}
    Then, the equation would be:
    \begin{align}
        \frac{\partial{\hat{g}_r}}{\partial g_r} &= \frac{(1-\beta_1)(V+\epsilon) - (\frac{(1-\beta_2)\beta_1 m_{r-1} + (1-\beta_2)(1-\beta_1)g_r}{V\sqrt{1-\beta_2^r}})g_r}{(1-\beta_1^r)(V+\epsilon)^2} \label{eq:ggg}
    \end{align}
    Given $\beta_1=0.9$, $\beta_2=0.999$ and $\epsilon$ is a very small value that can be ignored, considering both large or small $r$, we approximate:
    \begin{align}
        (V + \epsilon) &\approx \alpha g_r
    \end{align}
    where $\alpha$ is a non-zero scaling variant. In the end, we have:
    \begin{align}
        \frac{\partial{\hat{g}_r}}{\partial g_r} &= \frac{-\beta_1 m_{r-1}}{\alpha (1-\beta_1^r)} \cdot \frac{1}{g_r^2} \label{eq:gggg}
    \end{align}
    See Appendix for more specific derivation process.

    In \eqref{eq:gggg}, the parameters $\alpha$ and $(1-\beta_1)$ are non-zero values, while $m_{r-1}$ is exclusively preserved within the confines of local clients. This specific configuration ensures that the original local gradients, $g_r$, are safeguarded against deep leakage from gradients by potential malicious attackers. Consequently, the proposed gradients stand-in, $\hat{g}_r$, effectively obstructs any attempts to deduce or reconstruct private training data from the local gradients. This safeguarding mechanism enhances the security of the data during the \ac{FL} training process, providing a robust defense against potential data privacy breaches.
\end{proof}

\subsection{AdaDefense in \ac{FL}}

\noindent\textbf{Federated Learning:}
The foundational method underpinning recent \ac{FL} techniques assumes the presence of multiple clients, represented as $\mathcal{C}$, each possessing their own local datasets $\mathcal{D}$. The learning task is defined by a model $\mathcal{F}$ with parameters $\omega$. The local gradient $g_i$ for each client $i$ is computed as follows:
\begin{align}
    g_i = \frac{1}{||d_i||} \nabla_\omega \sum_j \mathcal{L}(\mathcal{F}(x^{(j)}; \omega), y^{(j)}) \quad \text{for all } i \in \mathcal{C},
\end{align}
where $||\cdot||$ denotes the L2 norm, used here to normalize the gradients by the size of the dataset $d_i$. In this formula, $x^{(j)}$ and $y^{(j)}$ are the input features and corresponding labels of the $j$-th example in the $i$-th local dataset.

After computing the local gradients, the server aggregates these gradients and performs an averaging process to update the global model parameters $\omega_{r}$. This process is mathematically represented as:
\begin{align}
    \omega_{r} = \omega_{r-1} - \frac{1}{|\mathcal{C}|} \sum_{i=1}^{|\mathcal{C}|} g_i,
\end{align}
where $|\mathcal{C}|$ denotes the total number of clients. It is assumed that all local datasets are of equal size, i.e., $||d_i|| = ||d_k||$ for any $d_i, d_k \in \mathcal{D}$. This assumption simplifies the model updating process by maintaining uniform influence for each client's gradient.

\noindent\textbf{AdaDefense:}
Emphasizing the innovative nature of AdaDefense is essential, as it integrates seamlessly into existing frameworks without altering the fundamental architecture of the model or the \ac{FL} strategy it supports. AdaDefense is designed to act as an unobtrusive extension, focusing on the manipulation of local gradients. This process involves the creation of gradients stand-in that is used in the global aggregation phase. In a standard \ac{FL} setup, the global model is initially crafted on a central server and then distributed to various clients. These clients perform secure local training and then send their local gradients stand-ins back to the server. At the server, these gradients stand-ins are aggregated to update and refine the global model. This cycle of local training and global updating repeats, with the server coordinating the aggregation of gradients to produce each new iteration of the global model. During this process, the server remains oblivious to all clients' first-order and second-order moments information. This omission is intentional and essential for security. This structure ensures that gradients stand-in cannot be exploited to reconstruct private, local training data. By maintaining this separation, AdaDefense strengthens the \ac{FL} system against potential data reconstruction attacks. Consequently, in every round of local training that follows, each client adjusts its model based on the updated global model distributed by the server, thus advancing the collective learning process without compromising on privacy or security. See Fig.~\ref{fig:ad_fl} for an illustration of AdaDefense in \ac{FL}.

\begin{figure*}[t!]
    \centering
    \includegraphics[width=0.9\linewidth]{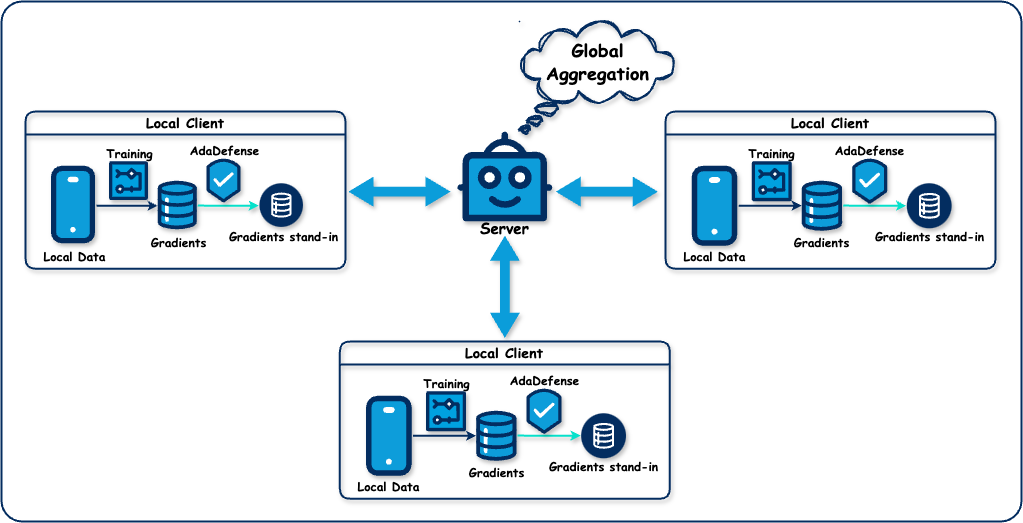}
    \caption{Illustration of AdaDefense in \ac{FL}}
    \label{fig:ad_fl}
\end{figure*}

\section{Experiments}
\label{sec:er}

In this section, we begin by describing the datasets and metrics used for benchmark evaluation. We then detail a series of experiments designed to evaluate the performance of \ac{DNN}s across two aspects. The first set of experiments investigates the impact of using the proposed gradients stand-in on the prediction accuracy. The second set assesses the effectiveness of this approach in defending against state-of-the-art gradient leakage attacks, specifically those involving \ac{GRNN}, \ac{IG} and \ac{GLAUS}.
\subsection{Benchmarks and Metrics}

In our study, we conducted experiments using four well-known public datasets: MNIST~\cite{lecun1998mnist}, CIFAR-10~\cite{krizhevsky2009learning}, CIFAR-100~\cite{krizhevsky2009learning}, and ILSVRC2012~\cite{deng2009imagenet}. To assess the quality of images generated by our method, we employed four evaluation metrics: \ac{MSE}, \ac{PSNR}, \ac{LPIPS}~\cite{zhang2018unreasonable}, and \ac{SSIM}~\cite{wang2004image}. To optimize the conditions for gradient leakage attacks and effectively evaluate the defense capabilities, we set the batch size to one. This configuration allows for a direct comparison between the reconstructed image and the original image. The \ac{MSE} metric quantifies the pixel-wise L2 norm difference between the reconstructed image $\hat{X}$ and the original image $X$, defined as $MSE(X, \hat{X}) = ||X - \hat{X}||_2$. \ac{PSNR}, another criterion for image quality assessment, is calculated using the formula $PSNR(X, \hat{X}) = 10 \cdot \lg(\frac{255^2}{MSE(X, \hat{X})})$, where a higher \ac{PSNR} value indicates a higher similarity between the two images. For perceptual similarity, we utilized the \ac{LPIPS} metric, which measures the similarity between two image patches using pre-trained networks such as \emph{VGGNet}~\cite{simonyan2014very} and \emph{AlexNet}~\cite{krizhevsky2012imagenet}. This metric aligns with human visual perception, where a lower \ac{LPIPS} score indicates greater perceptual similarity between the compared images. The \ac{SSIM} index evaluates the change in image structure, where a higher \ac{SSIM} value suggests less distortion and thus a better quality of the reconstructed image. All neural network models were implemented using the PyTorch framework~\cite{paszke2019pytorch}, ensuring robust and efficient computation.
\subsection{Model Performance}

\begin{table*}[!ht]
\setlength{\tabcolsep}{8pt}
\begin{center}
\caption{Testing accuracy (\%) of models trained under various configurations. The highest accuracy achieved is highlighted in red, while the second highest is denoted in blue. Here, ``C-10" and ``C-100" refer to CIFAR-10 and CIFAR-100 datasets, respectively. ``AD" stands for the AdaDefense module. The accuracy figures listed in the `Reference' row are sourced from corresponding papers.}
\label{tab:accuracy}
\begin{tabular}{c c|c|c c|c c|c c|c c|c c}
\hline
\multicolumn{2}{c|}{\textbf{Model}} & \makecell[c]{\emph{LeNet}\\(32*32)} & \multicolumn{2}{c|}{\makecell[c]{\emph{ResNet-20}\\(32*32)}} & \multicolumn{2}{c|}{\makecell[c]{\emph{ResNet-32}\\(32*32)}} & \multicolumn{2}{c|}{\makecell[c]{\emph{ResNet-18}\\(224*224)}} & \multicolumn{2}{c|}{\makecell[c]{\emph{ResNet-34}\\(224*224)}} & \multicolumn{2}{c}{\makecell[c]{\emph{VGG-16}\\(224*224)}} \\
\hline
\multicolumn{2}{c|}{\textbf{Dataset}} & MNIST & C-10 & C-100 & C-10 & C-100 & C-10 & C-100 & C-10 & C-100 & C-10 & C-100 \\
\hline
\textbf{Centralized} && 98.09 & {\color{red}91.63} & {\color{red}67.59} & {\color{blue}92.34} & {\color{red}70.35} & {\color{red}91.62} & {\color{red}72.15} & {\color{red}92.20} & {\color{red}73.21} & {\color{red}89.13} & {\color{red}63.23} \\
\hline
\textbf{FedAvg} && {\color{blue}98.14} & 91.20 & 58.58 & 91.37 & 61.91 & {\color{blue}89.50} & 68.59 & 89.27 & 68.30 & {\color{blue}88.13} & {\color{blue}61.91} \\
\hline
\textbf{FedAvg w/ AD} && 97.80 & 89.15 & {\color{blue}61.81} & 90.20 & {\color{blue}63.64} & 89.29 & {\color{blue}68.76} & {\color{blue}89.38} & {\color{blue}68.38} & 87.13 & 58.32 \\
\hline
\textbf{Reference} && \makecell[c]{{\color{red}99.05}\\\cite{lecun1998gradient}} & \makecell[c]{{\color{blue}91.25}\\\cite{he2016deep}} & - & \makecell[c]{{\color{red}92.49}\\\cite{he2016deep}} & - & - & - & - & - & - & - \\
\hline
\end{tabular}
\end{center}
\end{table*}

To evaluate the impact on performance, we developed several distinct training methodologies. The first approach involved centralized training. The second approach was collaborative training within the \ac{FedAvg} framework. The last the approach was also the \ac{FedAvg}, but involved gradients stand-in. For the training of image classification models, we selected six foundational network architectures: \emph{LeNet}~\cite{lecun1998gradient}, \emph{ResNet-20}, \emph{ResNet-32}, \emph{ResNet-18}, \emph{ResNet-34}, ~\cite{he2016deep} and \emph{VGGNet-16}~\cite{simonyan2014very} on different benchmarks including MNIST, CIFAR-10 and CIFAR-100 with image resolutions from $32*32$ to $224*224$. We systematically trained a total of thirty-three models across eleven experimental setups using both centralized and federated training strategies. This comprehensive approach allowed us to compare the performance of standard networks against those with gradients stand-in. The results of these comparisons are systematically detailed in Table~\ref{tab:accuracy}, providing a clear overview of the impact of the gradients stand-in on model performance in different training environments.

Centralized training consistently produces the most favorable results. The models appear to have a high performance on the MNIST and CIFAR-10 datasets, with accuracy slightly decreasing for the CIFAR-100 dataset. This pattern suggests that CIFAR-100 has increased complexity, due to a larger number of classes, poses a greater challenge for the models. For example, the implementation of \emph{LeNet} on the MNIST dataset results in a remarkable accuracy of $99.05\%$, as reported in~\cite{lecun1998gradient}. Similarly, \emph{ResNet-20}, \emph{ResNet-32}, \emph{ResNet-18}, \emph{ResNet-34} and \emph{VGG-16} achieves their highest accuracies on CIFAR-10 and CIFAR-100 datasets, recording $91.63\%$ and $67.59\%$, $92.49\%$ and $70.35\%$, $91.62\%$ and $72.15\%$, $92.20\%$ and $73.21\%$, $89.13\%$ and $63.23\%$, respectively.
When employing \ac{FedAvg}, there is a noticeable trend where most models maintain a performance close to the centralized training, which showcases the effectiveness of \ac{FL} in preserving model accuracy even when the data is decentralized. \ac{FedAvg} gains four out of eleven second highest accuracies for \emph{LeNet} on MNIST, \emph{ResNet-18} on CIFAR-10, \emph{VGG-16} on CIFAR-10 and CIFAR-100. The incorporation of the gradients stand-in into the training framework generally does not adversely affect the model's inference performance. In some cases, models equipped with the gradients stand-in even surpass the performance of \ac{FedAvg}, achieving five out of eleven second highest accuracies. Within the AdaDefense framework, \emph{ResNet-20} on CIFAR-100 achieves a testing accuracy of $61.81\%$, which is significantly higher by $5.51\%$ than the conventional \ac{FedAvg} model, which only reaches $58.58\%$.

\subsection{Defense Performance}

In evaluating the robustness of defense mechanisms against gradient leakage attacks, our study incorporated four advanced attack methodologies. These included \ac{GRNN}~\cite{ren2022grnn}, \ac{IG}~\cite{geiping2020inverting} and \ac{GLAUS}~\cite{yang2023reveal}. For the datasets MNIST, CIFAR-10, and CIFAR-100, experiments were conducted using a batch size of one and an image resolution of $32 * 32$. However, for the ILSVRC2012 dataset, a higher resolution of $256 * 256$ was selected rather than the usual $224 * 224$. This adjustment was necessary because \ac{GRNN} can only process image resolutions that are powers of two. To accommodate this requirement, images of lower resolutions were upscaled through linear interpolation to match the required dimensions. This approach ensures that the integrity and comparability of the results across different datasets and resolution settings are maintained.

\begin{table*}[!ht]
\setlength{\tabcolsep}{5pt}
\begin{center}
\caption{Comparison of image reconstruction using \ac{GRNN} and \ac{IG} without and with AdaDefense.}
\label{tab:dp}
\begin{tabular}{c | c | c c c | c c c | c c c }
\hline
\multicolumn{2}{c|}{Model} & \multicolumn{3}{c|}{\makecell[c]{\emph{LeNet}\\(32*32)}} & \multicolumn{3}{c|}{\makecell[c]{\emph{ResNet-20}\\(32*32)}} & \multicolumn{3}{c}{\makecell[c]{\emph{ResNet-18}\\(256*256)}} \\ 
\hline
\multicolumn{2}{c|}{Dataset} & MNIST & C-10 & C-100 & MNIST & C-10 & C-100 & C-10 & C-100 & ILSVRC \\
\hline
\multirow{13}{*}{GRNN} & True & \makecell*[c]{\includegraphics[width=0.07\linewidth]{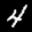}} & \makecell*[c]{\includegraphics[width=0.07\linewidth]{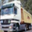}} & \makecell*[c]{\includegraphics[width=0.07\linewidth]{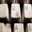}} & \makecell*[c]{\includegraphics[width=0.07\linewidth]{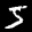}} & \makecell*[c]{\includegraphics[width=0.07\linewidth]{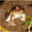}} & \makecell*[c]{\includegraphics[width=0.07\linewidth]{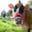}} & \makecell*[c]{\includegraphics[width=0.07\linewidth]{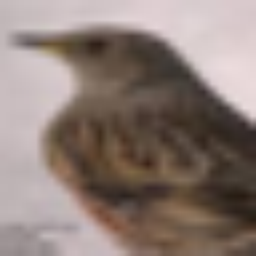}} & \makecell*[c]{\includegraphics[width=0.07\linewidth]{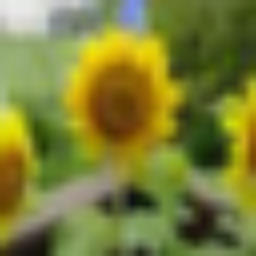}} & \makecell*[c]{\includegraphics[width=0.07\linewidth]{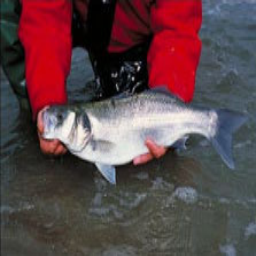}}\\
&     & four & truck & keyboard & five & frog & cattle & bird & sunflower & tench\\

& w/o & \makecell*[c]{\includegraphics[width=0.07\linewidth]{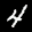}} & \makecell*[c]{\includegraphics[width=0.07\linewidth]{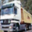}} & \makecell*[c]{\includegraphics[width=0.07\linewidth]{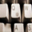}} & \makecell*[c]{\includegraphics[width=0.07\linewidth]{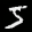}} & \makecell*[c]{\includegraphics[width=0.07\linewidth]{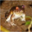}} & \makecell*[c]{\includegraphics[width=0.07\linewidth]{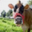}} & \makecell*[c]{\includegraphics[width=0.07\linewidth]{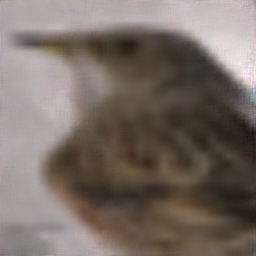}} & \makecell*[c]{\includegraphics[width=0.07\linewidth]{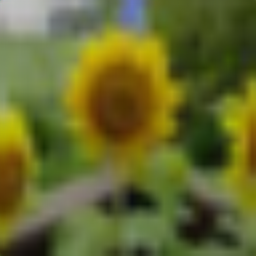}} & \makecell*[c]{\includegraphics[width=0.07\linewidth]{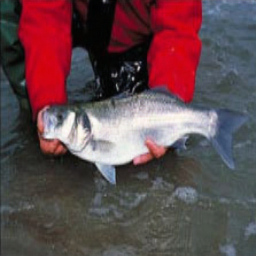}} \\
&     & four & truck & keyboard & five & frog & cattle & bird & sunflower & tench\\
& w/  & \makecell*[c]{\includegraphics[width=0.07\linewidth]{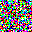}} & \makecell*[c]{\includegraphics[width=0.07\linewidth]{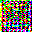}} & \makecell*[c]{\includegraphics[width=0.07\linewidth]{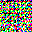}} & \makecell*[c]{\includegraphics[width=0.07\linewidth]{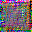}} & \makecell*[c]{\includegraphics[width=0.07\linewidth]{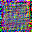}} & \makecell*[c]{\includegraphics[width=0.07\linewidth]{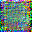}} & \makecell*[c]{\includegraphics[width=0.07\linewidth]{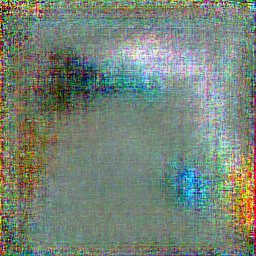}} & \makecell*[c]{\includegraphics[width=0.07\linewidth]{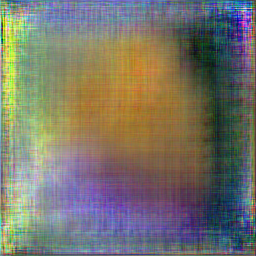}} & \makecell*[c]{\includegraphics[width=0.07\linewidth]{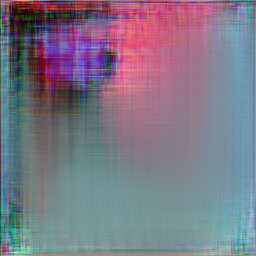}}\\
&     & zero & ship & bowl & four & horse & mushroom & dog & oak tree & soccer ball\\

\hline
\multirow{13}{*}{IG} & True & \makecell*[c]{\includegraphics[width=0.07\linewidth]{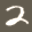}} & \makecell*[c]{\includegraphics[width=0.07\linewidth]{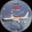}} & \makecell*[c]{\includegraphics[width=0.07\linewidth]{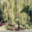}} & \makecell*[c]{\includegraphics[width=0.07\linewidth]{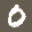}} & \makecell*[c]{\includegraphics[width=0.07\linewidth]{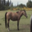}} & \makecell*[c]{\includegraphics[width=0.07\linewidth]{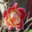}} & \makecell*[c]{\includegraphics[width=0.07\linewidth]{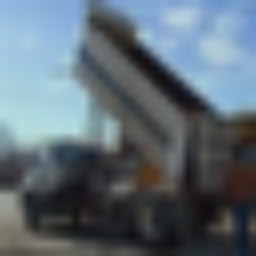}} & \makecell*[c]{\includegraphics[width=0.07\linewidth]{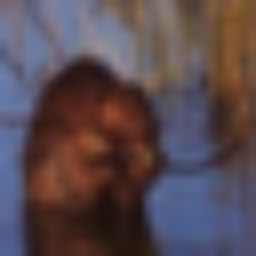}} & \makecell*[c]{\includegraphics[width=0.07\linewidth]{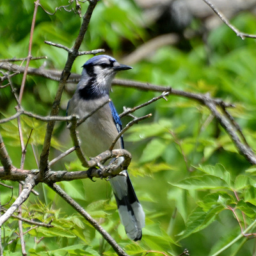}} \\
&     & two & airplane & willow & zero & horse & rose & truck & beaver & jay\\
& w/o & \makecell*[c]{\includegraphics[width=0.07\linewidth]{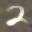}} & \makecell*[c]{\includegraphics[width=0.07\linewidth]{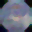}} & \makecell*[c]{\includegraphics[width=0.07\linewidth]{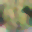}} & \makecell*[c]{\includegraphics[width=0.07\linewidth]{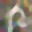}} & \makecell*[c]{\includegraphics[width=0.07\linewidth]{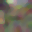}} & \makecell*[c]{\includegraphics[width=0.07\linewidth]{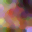}} & \makecell*[c]{\includegraphics[width=0.07\linewidth]{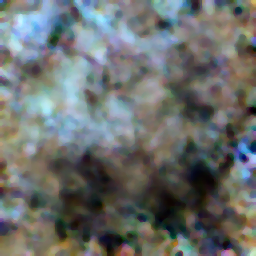}} & \makecell*[c]{\includegraphics[width=0.07\linewidth]{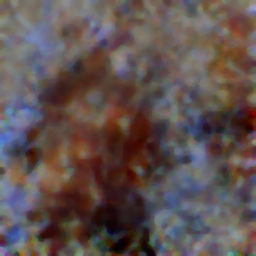}} & \makecell*[c]{\includegraphics[width=0.07\linewidth]{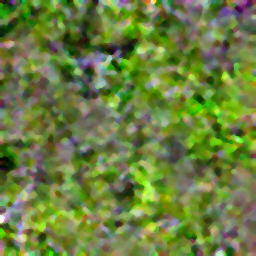}} \\
&     & - & - & - & - & - & - & - & - & -\\
& w/  & \makecell*[c]{\includegraphics[width=0.07\linewidth]{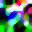}} & \makecell*[c]{\includegraphics[width=0.07\linewidth]{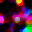}} & \makecell*[c]{\includegraphics[width=0.07\linewidth]{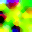}} & \makecell*[c]{\includegraphics[width=0.07\linewidth]{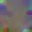}} & \makecell*[c]{\includegraphics[width=0.07\linewidth]{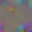}} & \makecell*[c]{\includegraphics[width=0.07\linewidth]{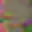}} & \makecell*[c]{\includegraphics[width=0.07\linewidth]{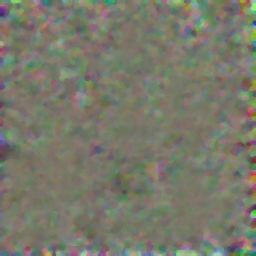}} & \makecell*[c]{\includegraphics[width=0.07\linewidth]{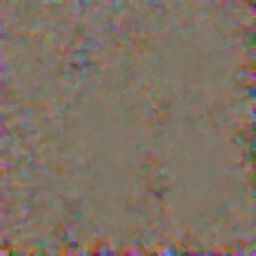}} & \makecell*[c]{\includegraphics[width=0.07\linewidth]{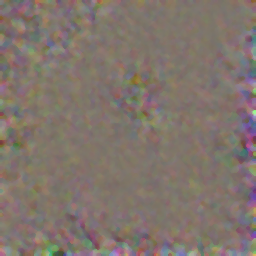}}\\
&     & - & - & - & - & - & - & - & - & -\\
\hline
\end{tabular}
\end{center}
\vspace{-7pt}
\end{table*}

\begin{table}[ht]
\setlength{\tabcolsep}{3pt}
\begin{center}
\caption{Image reconstruction for \ac{GLAUS} without and with AdaDefense.}
\label{tab:dp-glaus}
\begin{tabular}{c | c | c c c c c}
\hline
\multicolumn{2}{c|}{Model} & \multicolumn{5}{c}{\makecell[c]{\emph{MNIST-CNN}\\(28*28)}} \\ 
\hline
\multicolumn{2}{c|}{Dataset} & \multicolumn{5}{c}{MNIST}\\
\hline
\multirow{13}{*}{GLAUS} & True & \makecell*[c]{\includegraphics[width=0.13\linewidth]{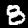}} & \makecell*[c]{\includegraphics[width=0.13\linewidth]{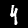}} & \makecell*[c]{\includegraphics[width=0.13\linewidth]{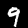}} & \makecell*[c]{\includegraphics[width=0.13\linewidth]{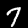}} & \makecell*[c]{\includegraphics[width=0.13\linewidth]{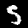}} \\
&     & eight & four & nine & seven & five\\
& w/o & \makecell*[c]{\includegraphics[width=0.13\linewidth]{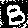}} & \makecell*[c]{\includegraphics[width=0.13\linewidth]{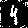}} & \makecell*[c]{\includegraphics[width=0.13\linewidth]{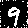}} & \makecell*[c]{\includegraphics[width=0.13\linewidth]{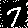}} & \makecell*[c]{\includegraphics[width=0.13\linewidth]{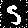}} \\
&     & eight & four & nine & seven & five\\
& w/  & \makecell*[c]{\includegraphics[width=0.13\linewidth]{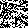}} & \makecell*[c]{\includegraphics[width=0.13\linewidth]{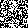}} & \makecell*[c]{\includegraphics[width=0.13\linewidth]{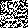}} & \makecell*[c]{\includegraphics[width=0.13\linewidth]{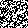}} & \makecell*[c]{\includegraphics[width=0.13\linewidth]{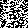}} \\
&     & five & five & five & five & five\\
\hline
\end{tabular}
\end{center}
\vspace{-7pt}
\end{table}

In Table~\ref{tab:dp}, we present qualitative results that demonstrate the comparative effectiveness of various defense mechanisms against targeted leakage attacks. These defenses include \ac{GRNN} and \ac{IG}, which have been implemented across different backbone networks and benchmark datasets. The table provides a clear view of how each defense strategy performs, allowing for an assessment of their efficacy in a controlled environment. 

As discussed in the paper~\cite{ren2022grnn}, \ac{DLG} is inadequate in several experimental scenarios, leading us to exclude its results from our analysis. Instead, we focused on \ac{GRNN}, which has proven more effective. Our proposed AdaDefense that efficiently safeguards against the leakage of true images during gradient-based reconstruction processes. In our experiments using \ac{GRNN} with a \emph{ResNet-18} model at an image resolution of $256 * 256$, the reconstructed images retained only minimal and non-identifiable details from the original images. This indicates a significant protection against the exposure of private data. For other architectures, such as \emph{LeNet} and \emph{ResNet-20}, the results were even more promising, with the reconstructed images being completely unrecognizable and distorted. On the other hand, the \ac{GRNN} is capable of revealing true labels from gradients. Consequently, we have included these labels in Table~\ref{tab:dp} for comprehensive analysis. It is evident from our observations that when models are trained using our AdaDefense module, no correct labels are inferred from gradients. This further indicates the effectiveness of AdaDefense in enhancing model security against gradient-based label inference attacks.

Furthermore, we compared the reconstruction capabilities of \ac{IG} with those of \ac{GRNN}. The findings confirmed that \ac{GRNN} outperforms \ac{IG} in safeguarding private training data across all tested configurations. Overall, our AdaDefense consistently demonstrated robust protection of private training information from gradient-based reconstructions in our experiments.

In the recent publication on the method \ac{GLAUS}, it is noted that the official code supports only a specially designed architecture, \emph{MNIST-CNN}, for use with the MNIST dataset. Attempts to adapt this method for other neural networks and datasets were unsuccessful. Consequently, the results presented in Table~\ref{tab:dp-glaus} are based solely on experiments conducted using the original, unmodified official code. This table highlights certain inherent shortcomings of \ac{GLAUS}, even in the absence of our defense method. These limitations are generally considered acceptable, given that \ac{GLAUS} is specifically tailored to attack \ac{UGS} secure aggregation methods within the \ac{FL} framework. Notably, when our AdaDefense module is integrated into the system, \ac{GLAUS} is rendered ineffective in all trials concerning image reconstruction and label inference.

\subsection{Quantitative Analysis}

\begin{table*}[!ht]
\setlength{\tabcolsep}{8pt}
\begin{center}
\caption{Quantitative comparison of \ac{GRNN} and \ac{IG} without and with AdaDefense module. The results are computed from the generated images and their associated true images.}
\label{tab:qc}
\begin{tabular}{c c c c >{\columncolor{blue!10}}c c >{\columncolor{red!10}}c c >{\columncolor{blue!10}}c c >{\columncolor{red!10}}c c >{\columncolor{blue!10}}c}
\hline
\multirow{2}{*}{\textbf{Method}} & \multirow{2}{*}{\textbf{Model}} & \multirow{2}{*}{\textbf{Dataset}} & \multicolumn{2}{c}{MSE$\downarrow$} & \multicolumn{2}{c}{PSNR$\uparrow$} & \multicolumn{2}{c}{LPIPS-V$\downarrow$} & \multicolumn{2}{c}{LPIPS-A$\downarrow$} & \multicolumn{2}{c}{SSIM$\uparrow$} \\
\cmidrule(r){4-5}\cmidrule(r){6-7}\cmidrule(r){8-9}\cmidrule(r){10-11}\cmidrule(r){12-13}
&&& w/o & w/ & w/o & w/ & w/o & w/ & w/o & w/ & w/o & w/ \\
\hline
\multirow{9}{*}{\textbf{GRNN}} & \multirow{3}{*}{\makecell[c]{\emph{LeNet}\\(32*32)}}
& MNIST & 0.65 & 173.44 & 50.11 & 25.74 & 0.05 & 0.75 & 0.01 & 0.56 & 1.000 & -0.003\\
&& C-10 & 1.97 & 143.19 & 45.28 & 26.57 & 0.00 & 0.56 & 0.00 & 0.35 & 0.997 & 0.000\\
&& C-100 & 1.79 & 145.36 & 45.85 & 26.51 & 0.00 & 0.55 & 0.00 & 0.38 & 0.998 & 0.003\\
\cline{3-13}
& \multirow{3}{*}{\makecell[c]{\emph{ResNet-20}\\(32*32)}}
& MNIST & 3.55 & 125.81 & 44.04 & 27.14 & 0.12 & 0.70 & 0.01 & 0.47 & 0.954 & 0.003\\
&& C-10 & 12.16 & 85.43 & 40.68 & 28.84 & 0.00 & 0.44 & 0.00 & 0.25 & 0.973 & 0.030\\
&& C-100 & 22.30 & 86.43 & 37.37 & 28.84 & 0.00 & 0.43 & 0.00 & 0.24 & 0.958 & 0.053\\
\cline{3-13}
& \multirow{3}{*}{\makecell[c]{\emph{ResNet-18}\\(256*256)}}
& C-10 & 17.36 & 70.76 & 37.55 & 29.69 & 0.01 & 0.43 & 0.01 & 0.57 & 0.955 & 0.151 \\
&& C-100 & 32.56 & 64.37 & 34.67 & 30.19 & 0.02 & 0.31 & 0.02 & 0.28 & 0.917 & 0.434 \\
&& ILSVRC & 18.82 & 59.81 & 37.31 & 30.41 & 0.05 & 0.41 & 0.04 & 0.39 & 0.891 & 0.212 \\
\hline
\multirow{9}{*}{\textbf{IG}} & \multirow{3}{*}{\makecell[c]{\emph{LeNet}\\(32*32)}}
& MNIST & 19.23 & 90.62 & 35.29 & 28.56 & 0.06 & 0.50 & 0.03 & 0.32 & 0.825 & 0.300 \\
&& C-10 & 22.16 & 101.09 & 34.68 & 28.08 & 0.18 & 0.63 & 0.06 & 0.30 & 0.637 & 0.060 \\
&& C-100 & 20.45 & 92.32 & 35.02 & 28.48 & 0.12 & 0.53 & 0.06 & 0.31 & 0.623 & 0.100 \\
\cline{3-13}
& \multirow{3}{*}{\makecell[c]{\emph{ResNet-20}\\(32*32)}}
& MNIST & 49.25 & 60.81 & 31.23 & 30.30 & 0.22 & 0.27 & 0.15 & 0.24 & 0.361 & 0.125 \\
&& C-10 & 39.35 & 52.31 & 32.27 & 30.99 & 0.21 & 0.27 & 0.16 & 0.16 & 0.299 & 0.098 \\
&& C-100 & 38.19 & 47.25 & 32.37 & 31.41 & 0.22 & 0.31 & 0.10 & 0.20 & 0.309 & 0.110 \\
\cline{3-13}
& \multirow{3}{*}{\makecell[c]{\emph{ResNet-18}\\(256*256)}}
& C-10 & 64.38 & 45.54 & 30.11 & 32.08 & 0.35 & 0.32 & 0.22 & 0.21 & 0.325 & 0.476 \\
&& C-100 & 78.44 & 63.13 & 29.29 & 30.19 & 0.37 & 0.28 & 0.30 & 0.27 & 0.355 & 0.510 \\
&& ILSVRC & 82.46 & 79.16 & 28.99 & 29.20 & 0.64 & 0.49 & 0.56 & 0.48 & 0.103 & 0.266 \\
\hline
\end{tabular}
\end{center}
\vspace{-10pt}
\end{table*}

To conduct a quantitative assessment of our proposed AdaDefense module in comparison to existing state-of-the-art methods for mitigating gradient leakage, we utilized four distinct evaluation metrics: \ac{MSE}, \ac{PSNR}, \ac{LPIPS} with both \emph{VGGNet} and \emph{AlexNet}, and \ac{SSIM}. The evaluation results, as detailed in Table~\ref{tab:qc}, were derived from comparisons between generated images and their respective original images. Our analysis consistently revealed that images produced by models incorporating the AdaDefense module exhibited significantly lower degrees of similarity when compared to those from models lacking this feature. This outcome demonstrates the efficacy of the AdaDefense module in providing robust defense against gradient leakage attacks.

To enhance the clarity of our analysis on the defensive capabilities of AdaDefense, we have concentrated our examination primarily on \ac{GRNN}. This focus is driven by our observations that \ac{GRNN} exhibits a significantly stronger attack capability compared to \ac{IG}. Through this targeted analysis, we aim to demonstrate more effectively the robustness of AdaDefense against more potent threats. Starting with \ac{MSE}, we see that in every case, the introduction of AdaDefense results in a dramatic increase in \ac{MSE}, which implies a higher error rate between generated and true images when the defense is active. For instance, the \ac{MSE} for the \emph{LeNet} model on the MNIST dataset jumps from $0.65$ without AdaDefense to $173.44$ with AdaDefense, representing an astonishing increase of approximately $26,600\%$. This trend of increased error rates with AdaDefense is consistent across all models and datasets, highlighting a significant protection of using AdaDefense in terms of gradients leakage.  Conversely, the impact on \ac{PSNR}, which measures image quality (higher is better), shows a general decline when AdaDefense is applied. For example, \ac{PSNR} for the \emph{ResNet-18} model on the ILSVRC dataset decreases from $37.31$ to $30.41$, a decrease of about $18.5\%$.
\ac{LPIPS}, evaluated using both \emph{VGGNet} and \emph{AlexNet}, measures perceptual similarity (lower is better). All the results from \ac{GRNN} perform significantly worse when the AdaDefense module is adapted into the model. Lastly, the \ac{SSIM} metric, which evaluates the structural similarity between the generated and true images, mostly shows a decline with AdaDefense. For the \emph{ResNet-18} model on the CIFAR-10 dataset, \ac{SSIM} decreases significantly from $0.955$ to $0.151$, about a $84.18\%$ reduction. This supports the view that AdaDefense can protect against specific vulnerabilities.
\section{Conclusion}
\label{sec:cc}

In this study, we introduce, AdaDefense, a novel defense mechanism against gradient leakage tailored for the \ac{FL} framework. The gist of AdaDefense is to use Adam to generate local gradients stand-ins for global aggregation, which effectively prevents the leakage of input data information through the local gradients. We provide a theoretical demonstration of how the gradients stand-in secures input data within the gradients and further validate its effectiveness through extensive experimentation with various advanced gradient leakage techniques. Our results, both quantitative and qualitative, indicate that integrating the proposed method into the model precludes the possibility of reconstructing input images from the publicly shared gradients. Furthermore, we evaluate the impact of the defense plugin on the model's classification accuracy, ensuring that the defense mechanism does not detract from the model’s performance. We are excited to make the implementation of AdaDefense available for further research and application. Moving forward, our focus will shift to exploring the intricate relationships between the gradients of different layers and their connection to the input data, which promises to enhance the robustness of \ac{FL} systems against various attack vectors.

\section*{Contribution Statement}

Yi Hu: Conceptualization, Formal analysis, Investigation, Methodology, Writing – original draft, Writing – review \& editing. 
Hanchi Ren: Conceptualization, Formal analysis, Investigation, Methodology, Supervision, Writing – review \& editing. 
Chen Hu: Writing – review. 
Yiming Li: Writing – review. 
Jingjing Deng: Conceptualization, Supervision, Writing – review. 
Xianghua Xie: Supervision


\bibliographystyle{IEEEtran}
\bibliography{ref}

\newpage
\appendix
\label{app:a}




\noindent\textbf{Proof of the Approximation:}

In the context of simplifying the derivative $\frac{\partial{\hat{g}_r}}{\partial g_r}$ and considering whether $V$ can be approximated by $\alpha g_r$, we need to analyze the terms carefully.

\begin{claim}
    For large values of $r$, the model has undergone sufficient training to approach convergence. At this stage, $V$ can be closely approximated by $\alpha g_r$, which simplifies the expression for the derivative $\frac{\partial \hat{g}_r}{\partial g_r}$.
\end{claim}
\begin{proof}
    Given the value $\beta_2 = 0.999$, evaluate the assumption that $V \approx \alpha g_r$, where:
    \begin{align}
        V = \sqrt{\frac{\beta_2 v_{r-1} + (1-\beta_2)g_r^2}{1-\beta_2^r}} \nonumber
    \end{align}
    Given the value of $\beta_2$, examine this approximation more critically. We first evaluate $\beta_2^r$. For larger $r$, $\beta_2^r$ becomes quite small because $\beta_2 = 0.999$ is very close to 1. The decay term $1 - \beta_2^r$ is crucial to understanding $V$'s behavior. To compute $1 - \beta_2^r$, say $r = 1000$ for example:
    \begin{align}
        \beta_2^{1000} = (0.999)^{1000} \nonumber
    \end{align}
    This can be computed as:
    \begin{align}
        \beta_2^{1000} \approx e^{1000ln(0.999)} \nonumber
    \end{align}
    Referring to Taylor expansion $ln(1-x) \approx -x$ for small $x$, we have:
    \begin{align}
        ln(0.999) \approx -0.001 \nonumber\\
        1000ln(0.999) \approx -1 \nonumber
    \end{align}
    Hence,
    \begin{align}
        \beta_2^{1000} \approx e^{-1} \approx 0.3679 \nonumber
    \end{align}
    Given the computation above, we have:
    \begin{align}
        V \approx \sqrt{\frac{\beta_2 v_{r-1} + (1-\beta_2)g_r^2}{0.6321}} \nonumber
    \end{align}
    Assuming $v_{r-1} \approx g_r^2$ for simplicity (which can be the case as gradients do not change dramatically over large iterations), then:
    \begin{align}
        V \approx \sqrt{\frac{0.999g_r^2 + 0.001g_r^2}{0.6321}} = \sqrt{\frac{g_r^2}{0.6321}} \approx \frac{g_r}{0.795} \approx 1.258g_r \nonumber
    \end{align}
    With $\beta_2 = 0.999$, the term $V$ is approximate $1.258g_r$ when assuming $v_{r-1} \approx g_r^2$ at larger $r$. So in this case, $\alpha \approx 1.258$
\end{proof}

\begin{claim}
    Given the condition of small $r$, the model does not converge effectively. Under this circumstance, the variable $V$ can be accurately approximated by $\alpha g_r$.
\end{claim}
\begin{proof}
    With the values $\beta_1 = 0.9$ and $\beta_2 = 0.999$ and considering small $r$, such as $r = 1$ or $r = 2$, we explore how $V$ behaves. This scenario is typical at the beginning of the training, where initial values for moving averages $m_0$ and $v_0$ are set to zero, and the initial gradient $g_0$ is randomly initialized. 
    
    Calculation for $r=1$:
    \begin{align}
        m_1 &= \beta_1 m_0 + (1-\beta_1) g_0 = 0.1 g_0 \nonumber \\
        v_1 &= \beta_2 v_0 + (1 - \beta_2) g_0^2 = 0.001 g_0^2 \nonumber \\
        \hat{m}_1 & = \frac{m_1}{1-\beta_1^1} = g_0 \nonumber \\
        \hat{v}_1 &= \frac{v_1}{1-\beta_2^1} = g_0^2 \nonumber \\
        V &= \sqrt{\hat{v}_1} = \sqrt{g_0^2} = g_0 \nonumber
    \end{align}

    Calculation for $r = 2$, assume $g_1 \approx g_0$ (for simplicity, assuming the gradient does not change too much):
    \begin{align}
        m_2 &= \beta_1 m_1 + (1-\beta_1) g_1 = 0.19 g_0 \nonumber \\
        v_2 &= \beta_2 v_1 + (1-\beta_2)g_1^2 = 0.002g_0^2 \nonumber \\
        \hat{m}_2 & = \frac{m_2}{1-\beta_1^2} \approx g_0 \nonumber \\
        \hat{v}_2 &= \frac{v_2}{1-\beta_2^2} \approx g_0^2 \nonumber \\
        V &= \sqrt{\hat{v}_2} \approx \sqrt{g_0^2} = g_0 \nonumber
    \end{align}

    With the values $\beta_1 = 0.9$ and $\beta_2 = 0.999$, the term $V$ is approximate $g_r$ when assuming $g_1 \approx g_0$ at the beginning of the training (small $r$). So in this case, $\alpha \approx 1$.
\end{proof}

\noindent\textbf{Derivation Process:}

Now, we are going to provide specific derivation process of \eqref{eq:gggg}. Given that $(V + \epsilon) \approx \alpha g_r$ and considering $\epsilon$ to be negligible for simplification, we will further simplify the derivative of $\hat{g}_r$ w.r.t. $g_r$:
\begin{align}
    \hat{g}_r &= \frac{\hat{m}_r}{\sqrt{\hat{v}_r} + \epsilon} \approx \frac{\hat{m}_r}{\sqrt{\hat{v}_r}} = \frac{\hat{m}_r}{V} = \frac{\hat{m}_r}{\alpha g_r} \nonumber
\end{align}
According to the quotient rule $(\frac{u}{v})' = \frac{u'v - uv'}{v^2}$, we have:
\begin{align}
    \hat{m}_r &= \frac{\beta_1 m_{r-1}+(1-\beta_1)g_r}{1-\beta_1^r} \nonumber \\
    \frac{\partial \hat{m}_r}{\partial g_r} &= \frac{1- \beta_1}{1-\beta_1^r} \nonumber
\end{align}
Hence,
\begin{align}
    \frac{\partial \hat{g}_r}{\partial g_r} &= \frac{(\frac{1- \beta_1}{1-\beta_1^r})(\alpha g_r) - (\frac{\beta_1 m_{r-1}+(1-\beta_1)g_r}{1-\beta_1^r}) (\alpha)}{(\alpha g_r)^2} \nonumber \\
    &= \frac{\frac{\alpha(1-\beta_1)g_r - \alpha(\beta_1 m_{r-1} + (1-\beta_1)g_r)}{1-\beta_1^r}}{{\alpha}^2 g_r^2} \nonumber \\
    &= \frac{\alpha(1-\beta_1)g_r - \alpha \beta_1 m_{r-1} - \alpha(1-\beta_1)g_r}{{\alpha}^2 g_r^2 (1-\beta_1^r)} \nonumber \\
    &= \frac{-\beta_1 m_{r-1}}{\alpha g_r^2 (1-\beta_1^r)}  \nonumber
\end{align}

\end{document}